%% file: main.tex
\documentclass[nohyperref]{article}

\usepackage{microtype}
\usepackage{graphicx}
\usepackage{subfigure}
\usepackage{booktabs} 

\usepackage{hyperref}

\usepackage[accepted]{icml2022}

\input{defines}

\usepackage[capitalize,noabbrev]{cleveref}

\theoremstyle{plain}
\newtheorem{theorem}{Theorem}[section]

\newtheorem{lemma}[theorem]{Lemma}
\newtheorem{corollary}[theorem]{Corollary}
\theoremstyle{definition}
\newtheorem{definition}[theorem]{Definition}

\theoremstyle{remark}

\usepackage[textsize=tiny]{todonotes}

\icmltitlerunning{Generalized Leverage Scores: Geometric Interpretation and Applications}

\begin{document}

\twocolumn[
\icmltitle{Generalized Leverage Scores: Geometric Interpretation and Applications}



\icmlsetsymbol{equal}{*}

\begin{icmlauthorlist}
\icmlauthor{Bruno Ordozgoiti}{qmul}
\icmlauthor{Antonis Matakos}{aalto}
\icmlauthor{Aristides Gionis}{kth}
\end{icmlauthorlist}

\icmlaffiliation{qmul}{School of Electronic Engineering and Computer Science, Queen Mary University of London, United Kingdom}
\icmlaffiliation{aalto}{Deparment of Computer Science, Aalto University, Finland}
\icmlaffiliation{kth}{Division of Theoretical Computer Science, KTH Royal Institute of Technology, Sweden}

\icmlcorrespondingauthor{Bruno Ordozgoiti}{b.[lastname]@qmul.ac.uk}

\icmlkeywords{Column Subset Selection, RandNLA, Leverage Scores, Matrix Computations, Numerical Linear Algebra}

\vskip 0.3in
]



\printAffiliationsAndNotice{}  

\begin{abstract}
In problems involving matrix computations, the concept of \emph{leverage} has found a large number of applications. In particular, \emph{leverage scores}, which relate the columns of a matrix to the subspaces spanned by its leading singular vectors, are helpful in revealing column subsets to approximately factorize a matrix with quality guarantees. As such, they provide a solid foundation for a variety of machine-learning methods. 
In this paper we extend the definition of leverage scores to relate the columns of a matrix to arbitrary subsets of singular vectors. We establish a precise connection between column and sin\-gular-vector subsets, by relating the concepts of leverage scores and principal angles between subspaces. We employ this result to design approximation algorithms with provable guarantees for two well-known problems: 
\emph{generalized column subset selection} and \emph{sparse canonical correlation analysis}. We run numerical experiments to provide further insight on the proposed methods. 
The novel bounds we derive improve our understanding of fundamental concepts in matrix approximations. In addition, our insights may serve as building blocks for further contributions.
\end{abstract}

\input{intro}
\input{preliminaries}
\input{caseuk}

\input{caseur}

\section{Applications}
\label{sec:applications}
\input{cssp}
\input{sparsecca}

\input{practical}

\input{numberofcolumns}

\input{experiments}

\input{related}

\section{Conclusions}
We have shown how the proposed \textit{generalized leverage scores},
along with our novel results, can be used to design approximation
algorithms for well-known problems. Our experimental results reveal
that the proposed methods have certain advantages over known
methods.
Given the fundamental nature of the concept of generalized leverage
scores and the numerous uses of standard leverage scores in the literature, 
we believe our
contributions are likely to be taken further by the community in
subsequent work. 

\medskip
\para{Acknowledgements.}
This research is supported by
Acade\-my of Finland projects 317085 and 325117,
the ERC Advanced Grant REBOUND (834862), 
the EC H2020 RIA project SoBigData++ (871042), and 
the Wallenberg AI, Autonomous Systems and Software Program (WASP) 
funded by the Knut and Alice Wallenberg Foundation.

\newpage
\bibliography{citations}
\bibliographystyle{icml2022}

\newpage
\appendix
\onecolumn
\input{appendix}


\end{document}

%% file: defines.tex
\usepackage{amsmath}
\usepackage{amssymb}
\usepackage{mathtools}
\usepackage{amsthm}
\usepackage{xspace}
\usepackage{xcolor}
\usepackage[font=small,skip=0pt]{caption}

\graphicspath{{.}{images/}}

\usepackage{trivfloat} 
\trivfloat{algorithm}

\newtheorem{property}{Property}
\newtheorem{problem}{Problem}
\newtheorem{fact}{Fact}

\renewenvironment{proof}[1][\proofname]{{\it\bfseries #1.}}{\qed}

\definecolor{dgreen}{rgb}{0.0, 0.35, 0.0}

\newif\ifshownotes
\shownotesfalse

\renewcommand{\cite}{\citep}

\newcommand{\para}[1]{\noindent{\bf #1}}

\newcommand{\scaler}[1]{\ensuremath{{\tilde\Sigma^{(#1)}}}\xspace}
\newcommand{\scaledvs}{\ensuremath{X}\xspace}

\newcommand{\topsigma}{\ensuremath{\sigma_\omega}}
\newcommand{\botsigma}{\ensuremath{\sigma_{\mu}}}
\newcommand{\condn}{\ensuremath{\eta}}
\newcommand{\range}{r}

\newcommand{\cssp}{{\sc CSS}\xspace}
\newcommand{\gcssp}{{\sc GCSS}\xspace}
\newcommand{\sparsecca}{{\sc SparseCCA}\xspace}
\newcommand{\greedy}{{\sc Greedy}\xspace}
\newcommand{\gls}{{\sc GLS}\xspace}

\newcommand{\wmiA}{\ensuremath{T}\xspace}
\newcommand{\wmiB}{\ensuremath{P}\xspace}
\newcommand{\wmiU}{\ensuremath{Y}\xspace}
\newcommand{\wmiV}{\ensuremath{Z}\xspace}

%% file: intro.tex
\section{Introduction}
When dealing with data in matrix form, it is often useful to find
compact low-dimensional approximations. One way to find such an approximation
is by computing a singular value decomposition (SVD), which offers a 
representation of a matrix in terms of a set of linearly-independent factors,
conveniently sorted in order of importance. 
The SVD is optimal the following sense: for any constant $k$, it reveals
the best rank-$k$ approximation of a matrix, as measured by a
family of matrix norms.

A drawback of SVD is that the resulting factors do not correspond directly to 
the rows or columns of the input matrix and lack an intuitive interpretation. 
To address this issue, many works in the literature
seek approximate factorizations in terms of elements of the matrix: 
instead of settling for a set of abstract factors, like those offered by the SVD, we aim to find a
subset of rows or columns that are particularly representative of the whole matrix.

Computing an SVD is a tractable problem and has been the subject of
extensive research. Thus, efficient methods exist to obtain 
an SVD to any desired level of precision~\cite{golub1996matrix}. 
When we wish to
select a subset of matrix elements instead, the task often becomes
computationally hard
under various natural objectives~\cite{civril2014column,shitov2017column}.
Therefore, research in this area has focused on finding efficient approximation
algorithms.

An example of a more intuitive matrix approximation is the 
\emph{column subset selection problem}~\cite{deshpande2006matrix,boutsidis2009improved},
which, given a matrix and a number~$k$, asks for the $k$ best columns, in the following sense:

\begin{problem}{\emph{Column subset selection (\cssp)}.}\label{problem:cssp}
Given a matrix $A \in \mathbb{R}^{m\times n}$ and a positive integer
$k$ smaller than the rank of~$A$, find a matrix $C$ comprised of $k$
columns of $A$ to minimize
\begin{equation}
\label{eq:cssp}
\|A-CC^+A\|_F^2,
\end{equation}
where $C^+$ is the Moore-Penrose pseudoinverse of
$C$.
\end{problem}
The matrix $CC^+A$ is the best Frobenius-norm approximation of $A$ in the column-space
of $C$ and is efficiently computable. 
It is easy to verify that the problem of maximizing $\|CC^+A\|_F^2$ is equivalent to Problem~\ref{problem:cssp},
 meaning that
the set of optimal solutions remain the same. 

\para{Leverage scores and good column subsets.} 
In finding good column subsets for the \cssp problem, 
\textit{leverage scores} have proved very useful. 
These scores --- to be precisely defined later --- relate the columns of a matrix to the
subspace spanned by its top singular vectors. Thus, they might reveal
whether a column subset is particularly representative. 
Leverage scores, and variants thereof, have been employed to design approximation algorithms 
for \cssp and related problems~\cite{drineas2006sampling,drineas2008relative,mahoney2009cur,boutsidis2009improved,papailiopoulos2014provable}. We give a more detailed overview of related work in the appendix.
Leverage scores can be traced back to the concept of \textit{statistical leverage}, 
long employed in statistics and the analysis of linear regression~\cite{chatterjee1986influential}. 

The techniques based on leverage scores, however, do not apply to
other matrix-approximation problems, such as the following
generalization of \cssp (in maximization form):

\begin{problem}{\emph{Generalized column subset selection (\gcssp)}.}
  \label{problem:gcssp}
Given two matrices $A\in \mathbb{R}^{m\times n}$, $B\in \mathbb{R}^{m\times p}$ and a positive integer
$k$ smaller than the rank of $A$, find a matrix $C$ comprised of $k$
columns of $A$ to maximize
\begin{equation}
\label{eq:gcssp}
\|CC^+B\|_F^2.
\end{equation}
\end{problem}

The reason that leverage scores are no longer helpful is that they
relate the columns of $A$ to its top-$k$ singular vectors, and these,
in turn, may not provide a good approximation of~$B$. 
The following question arises naturally: 
\emph{is there useful information elsewhere among the
singular subspaces of $A$?}

We
will show that, indeed, such information can be drawn from subspaces
spanned by arbitrary sets of singular vectors. This information, in
the form of leverage scores, will be useful to tackle problems like 
\gcssp. 


\subsection{Our Contributions.}
In this paper we introduce the concept of \textit{generalized leverage scores}, 
which relate the columns of a matrix to the subspace
spanned by an arbitrary subset of its singular vectors. 
Our definition, which provides a natural extension of leverage scores, 
can be employed to obtain approximation algorithms for a
variety of problems. Furthermore, the analysis leading to these
algorithms yields insightful results of independent interest, relating
leverage scores and the concept of angles between subspaces. 
In more detail, we make the following contributions.

\para{Generalized leverage scores and connections to principal angles:}
We introduce the concept of generalized leverage scores,
defined in terms of an arbitrary set of singular vectors, 
as contrasted to the standard leverage scores, 
which are defined with respect to the leading singular vectors. 
We show how this generalization enables the
application of leverage-score-based techniques to a new array of problems, 
leading to algorithms with approximation guarantees.

The cornerstone of these results is a novel inequality that gives a geometric interpretation of the generalized leverage scores, by relating them
to the principal angles between matrix columns and singular vectors.
The special case of standard leverage scores,
which is of independent interest, 
will be treated separately and proved using different arguments. 
We believe that our results, which relate two fundamental quantities, 
solidify our understanding of how leverage scores connect matrix columns and singular vectors.

\para{Applications:} We showcase the applicability
of the generalized leverage scores by providing approximation algorithms for two
well-known problems:

For \gcssp (Problem~\ref{problem:gcssp}), we show that choosing singular
vectors and generalized leverage scores to cover a
$(1-\epsilon)$-factor of readily accessible quantities, we find a
column submatrix $C$ of $A$ such that $\|CC^+B\|_F^2 \geq
(1-\epsilon)^2\|B\|_F^2$. This is akin to a known result
 for \cssp~\cite{papailiopoulos2014provable}, and complementary to
the only other known bound --- to our knowledge --- for \gcssp~\cite{bhaskara2016greedy}.
In contrast to the result of~\citet{bhaskara2016greedy}, the number of
columns chosen by our algorithm does not depend on the smallest singular value of the
optimal subset, which is unknown. Instead, it depends on the decay
of the generalized leverage scores, which is trivially computable. 
In particular, if the scores follow a power-law
decay, $\mathcal{O}\!\left(\frac{\topsigma^2 k}{\botsigma^2\epsilon}\right)$ columns suffice, 
where $\topsigma,\botsigma$ are a choice of singular values that
depends on the approximation constant $\epsilon$ --- the user, thus,
has control over this ratio.

For sparse canonical correlation analysis (\sparsecca) \cite{hardoon2011sparse,uurtio2017tutorial} 
we give similar results (see Section~\ref{sec:sparsecca}). 
Given input matrices $A$ and $B$, we use a similar criterion as before to find column subsets of both matrices whose
canonical correlations add up to a $(1-\epsilon)^4$ factor of the
total canonical correlations between $A$ and $B$. 
We are not aware of other algorithms
that give guarantees for \sparsecca in these terms.


The cost of the proposed algorithms is dominated by
the task of finding an SVD. Thus, our approaches can take full
advantage of the plethora of existing techniques for this purpose, as
well as fast approximate methods.

%% file: preliminaries.tex
\section{Preliminaries}

We will use upper and lowercase letters for matrices and vectors,
respectively, as in $A$ and $x$. We will also use uppercase letters
for sets. Context will be sufficient to tell sets and matrices apart.
For a matrix $A$ we write $A_{i,:}$ to denote its $i$-th row.
We write $A^+$ to denote the pseudoinverse of~$A$.

As in other works involving leverage scores, the singular value
decomposition will make frequent appearances in this paper.
We use the subindex $k$, as in $U_k$ and $V_k$, to denote the ``truncated''
submatrices of a singular value decomposition, obtained by retaining the
first $k$ singular vectors only. Analogously, given a set $R$ of natural numbers,
$U_R$ and $V_R$ will denote the matrices obtained by retaining the singular vectors indexed by $R$. 
For instance, note that $U_{k}=U_{\{1, \dots, k\}}$.

Chief among our cast of characters are the \textit{leverage scores}.

\begin{definition}
  Given a matrix $A$ with singular value decomposition $A=U\Sigma
  V^T$, the rank-$k$ leverage score of the $i$-th column of $A$ is
  defined as $\|(V_k)_{i,:}\|_2^2$. 
\end{definition}

Given a matrix $A$, we will denote the $i$-th diagonal element of
$AA^+$ as $\ell_i(A)$. This quantity is sometimes known in the statistics
literature as \textit{statistical leverage} of the $i$-th row of~$A$, 
and as we will see, it is connected to the leverage scores. Even though this choice
of nomenclature may occasionally lead to confusion, we retain it for
historical reasons. We ask the reader to remark the distinction in
this text between \textit{statistical leverages} and \textit{leverage
  scores}.

Also relevant is the concept of \textit{principal angles} between
subspaces, defined as
follows~\cite{bjorck1973numerical}:
\begin{definition}
Let $F$ and $G$ be subspaces of $\mathbb R^n$. Let the columns of matrices
$Q$ and $T$ be orthonormal bases of $F$ and~$G$, respectively. 
The principal angles between $F$ and $G$ are 
the angles whose cosines are the singular values of~$Q^TT$.
\end{definition}
We will often speak of angles between two matrices to mean the angles
between the subspaces spanned by their columns.

We make use of projection matrices and their
properties.
\begin{definition}
  A matrix $H$ is said to be a projection matrix (or projector) if $H^2=H$.
\end{definition}
For any projection matrix $H$, we have:
\begin{property}
  \label{proph_norms}
  $\sum_{j} H_{ij}^2 = H_{ii}$.
\end{property}
\begin{property}
  \label{proph_invariant}
  $H$ is invariant with respect to the basis of the space onto which it projects.
\end{property}
\begin{property}
  \label{proph_trace}
  $tr(H)=r$, where $r$ is the dimension of the space onto which $H$ projects.
\end{property}

%% file: caseuk.tex
\section{Leverage Scores and the Top-$k$ Singular Subspace}
\label{sec:caseuk}

In this section we present our first result, which relates the two
fundamental quantities defined in the previous section: 
the leverage scores and the principal angles between subspaces.
Even though this is a special case of the result in
Section~\ref{sec:caseur}, we treat it separately for two reasons:
first, it is of independent interest and, to the best of our knowledge, novel; 
second, the proof techniques will be useful --- but insufficient --- 
for the more general results of Section~\ref{sec:caseur}, 
so this case serves as a natural, more accessible, preamble.


To state our result, we will rely on two easily verified facts, given below. 
Here, we consider matrices $A=U\Sigma V^T$ and $C=AS$, 
where $S$ is a column selection matrix (binary matrix with unit-norm columns), 
and thus $C$ is comprised of a column subset of $A$. 
\begin{fact}
\label{fact:1}
The sum of the leverage scores of the columns of $C=AS$ is equal to $\|V_k^TS\|_F^2$.
\end{fact}
\begin{fact}
\label{fact:2}
The sum of the squared cosines of the principal angles between $C$ and
 $U_k$ is equal to $\|CC^+U_k\|_F^2$. 
\end{fact}
To verify Fact~\ref{fact:2}, note that if $U_C$ is an orthonormal basis for $C$, 
then $\|U_C^TU_k\|_F^2=\|U_CU_C^TU_k\|_F^2=\|U_CU_C^+U_k\|_F^2=\|CC^+U_k\|_F^2$. 
The last equality relies on Property~\ref{proph_invariant}.

We now state the main result of this section.
\begin{theorem}
 \label{theorem_angles_ls}
 Consider a matrix $A \in \mathbb R^{m\times n}$ and its singular value decomposition $A=U\Sigma V^T$. 
 Consider a column sampling matrix $S \in \mathbb R^{n\times r}$ and write $C=AS$. Then
 \[
 \|CC^+U_k\|_F^2 \geq \|V_k^TS\|_F^2.
 \]
\end{theorem}
In words, the sum of the leverage scores of a column subset provides a
lower bound for the sum of the cosines of the principal angles between two
subspaces: the one spanned by said column subset and the one spanned
by the top-$k$ left singular vectors. We believe this result provides
new insight on how the leverage scores connect column subsets and the
top-$k$ singular vectors. 

In our proof of Theorem~\ref{theorem_angles_ls} we will make use of
the following technical result, which characterizes the change in the
statistical leverages of a matrix upon multiplication of its rows by
scalars. 
Recall that $\ell_i(A)$ is the $i$-th diagonal element of $AA^+$.

\begin{lemma}
 \label{lemma_rankone}
 Consider a matrix $X\in \mathbb R^{n\times k}$ of rank
 $k$. Additionally, consider a non-negative real number $\alpha$ and
 a diagonal matrix $\scaler{i}$ defined as follows:
 $\scaler{i}_{ii}=\alpha$, $\scaler{i}_{jj}=1$ for all $j\neq i$.
 We write $x_j=X_{j,:}^T$, for any $j$. Then
 \begin{align*}
 &\ell_j(\scaler{i}X) = \ell_j(X)
 -\frac{(\alpha^2-1)(x_j^T(\scaledvs^T\scaledvs)^{-1}x_i)^2}{1+
  (\alpha^2-1)\ell_i(X)}, 
\end{align*} 
and in particular,
\begin{align*}
\ell_i(\scaler{i}X)=\frac{\alpha^2\ell_i(X)}{1+(\alpha^2-1)\ell_i(X)}.
\end{align*}
%
\end{lemma}
\input{proof_rankone}

\begin{proof}[Proof of Theorem~\ref{theorem_angles_ls}]
We will start by rewriting the quantities in question into a more
convenient form.
Note that $C=U\Sigma V^TS$. In addition, $(U\Sigma V^TS)^+=(\Sigma
V^TS)^+U^T$. Thus, we can write
\begin{align}
 \label{eq:sigma_vts}
 \|CC^+U_k\|_F^2 &= \|U\Sigma V^TS(\Sigma V^TS)^+U^TU_k\|_F^2 \nonumber
 \\ & = \left \|\Sigma V^TS(\Sigma
 V^TS)^+\left(\begin{array}{c}I_k\\ 0 \end{array}\right)
 \right\|_F^2 .
\end{align}
Here, $I_k$ is the $k\times k$ identity matrix.
The last equality follows because $U$ is orthogonal. 
The resulting quantity is the squared Frobenius norm of the 
matrix composed of the first $k$
columns of the projector $\Sigma V^TS(\Sigma V^TS)^+$. 
By Property~\ref{proph_norms},
any projector $H$ satisfies $H_{ii}=\sum_{j}H_{ji}^2$. 
In words, each diagonal entry equals the
squared norm of the corresponding column. 
Thus, the quantity in Equation~(\ref{eq:sigma_vts}) 
equals the sum of the \textit{statistical leverages} 
of the top $k$ rows of $\Sigma V^TS$.
We now turn to $V_k^TS$.
\begin{align}
\|V_k^TS\|_F^2 &= \|S^TV_k\|_F^2 = \left
\|S^TV\left(\begin{array}{c}I_k\\ 0 \end{array}\right) \right\|_F^2 \nonumber
\\ & = \left
\|V^TS(V^TS)^T\left(\begin{array}{c}I_k\\ 0 \end{array}\right)
\right\|_F^2 \label{eq:vts_1}
\\ & = \left
\|V^TS(V^TS)^+\left(\begin{array}{c}I_k\\ 0 \end{array}\right)
\right\|_F^2. \label{eq:vts_2}
\end{align}

Equations~(\ref{eq:vts_1}) and (\ref{eq:vts_2}) hold because $V^TS$ is comprised of orthonormal
columns, which means that the norm is unaffected by its presence as a
left-multiplying factor, and its pseudoinverse is equal to its transpose.

Equations~(\ref{eq:sigma_vts}) and (\ref{eq:vts_2}) suggest that we are
interested in analyzing the change in statistical leverage of  $V^TS$ 
when pre\-multiplied by diagonal matrix $\Sigma$. In
particular, we want to show that the sum of the statistical leverages
of the top $k$ rows of $V^TS$ does not decrease by this
multiplication. 

We will now make use of Property~\ref{proph_invariant} to coerce $\Sigma$
into a more favorable form.
In particular, we define $\tilde \Sigma = \sigma_k^{-1}\Sigma$. 
This ensures that $\sigma_i(\tilde \Sigma) \geq 1$, for $1\leq i\leq k$ and 
$\sigma_i(\tilde \Sigma)\leq 1$, for $i>k$.
Observe that by Property~\ref{proph_invariant}, $\Sigma V^TS(\Sigma V^TS)^+=\tilde\Sigma
V^TS(\tilde\Sigma V^TS)^+$.

To understand how $\tilde\Sigma$ affects the statistical leverages, we
will consider scaling each row separately. In particular, we
define $\tilde\Sigma^{(i)}$ as the diagonal matrix satisfying
$\tilde\Sigma_{ii}^{(i)}=\tilde\Sigma_{ii}$ and
$\tilde\Sigma_{jj}^{(i)}=1$, for $j\neq i$. 
Note that the matrices $V^TS$
and $\tilde\Sigma_{jj}^{(i)}$ satisfy the conditions of
Lemma~\ref{lemma_rankone}. 

Let $\ell_j(\tilde\Sigma^{(i)}V^TS)$ be the statistical leverage of the $j$-th row of $\tilde\Sigma^{(i)}V^TS$.
%
By Lemma~\ref{lemma_rankone} we have that

($i$) if $\tilde\sigma_i\geq 1$ then $\ell_i(\tilde\Sigma^{(i)}V^TS)\geq \ell_i(V^TS)$ and
 $\ell_j(\tilde\Sigma^{(i)}V^TS)\leq \ell_j(V^TS)$ for $j\neq i$; and

($ii$) if $\tilde\sigma_i\leq 1$ then $\ell_i(\tilde\Sigma^{(i)}V^TS)\leq \ell_i(V^TS)$ and
 $\ell_j(\tilde\Sigma^{(i)}V^TS)\geq \ell_j(V^TS)$ for $j\neq i$.

We will now analyze the effect of scaling all rows.
For convenience, we define matrices resulting from successively scaling the rows
of $V^TS$:
\[
\scaledvs^{(i)} = \left(\prod_{j=1}^i\tilde\Sigma^{(j)}\right)V^TS.
\]
That is, $\scaledvs^{(i)}$ is simply the matrix obtained by scaling
the first $i$ rows of $V^TS$ by the corresponding entries of
$\tilde\Sigma$.
From our discussion above we easily 
conclude that in the case of $\scaledvs^{(k)}$, the sum of the
statistical leverages of the bottom $n-k$ rows has not
increased. Furthermore, upon
successive left multiplication by $\tilde\Sigma^{(k+1)} \dots
\tilde\Sigma^{(n)}$ to obtain $\tilde\Sigma V^TS$, said sum cannot
increase. This is because $\sigma_i\leq 1$ for $j>k$. Finally,
we invoke Property~\ref{proph_trace} to indicate that 
$\sum_j\ell_j(V^TS)=\sum_j\ell_j(\scaledvs^{(i)})$ for any $i$, as scaling a row
cannot affect the rank, and so the dimension of the subspace spanned
by these matrices is the same.

Thus, as the sum of statistical leverages remains constant and the
bottom $n-k$ have not increased, we conclude that the sum
of the top $k$ statistical leverages has not decreased.

This concludes our proof.
\end{proof}

%% file: proof_rankone.tex
\begin{proof} 
~First, observe that since $X$ is full column rank, $\scaledvs\scaledvs^+ =
\scaledvs(\scaledvs^T\scaledvs)^{-1}\scaledvs^T$. Note also that
scaling one row of $\scaledvs$ through multiplication by a scalar 
is a rank-1 update of $\scaledvs$. In particular, let $x_i^T$ be the
$i$-th row of $\scaledvs$, and consider the product
$\tilde\Sigma^{(i)}\scaledvs$. It can be verified that
\[
(\tilde\Sigma^{(i)}\scaledvs)^T\tilde\Sigma^{(i)}\scaledvs =
\scaledvs^T\scaledvs + (\alpha^2-1)x_ix_i^T.
\]
The statistical leverage of the $j$-th row of
$\tilde\Sigma^{(i)}\scaledvs$ can thus be computed as
\[
\hat x_j^T(\scaledvs^T\scaledvs + (\alpha^2-1)x_ix_i^T)^{-1}\hat
x_j,
\]
where $\hat x_j^T$ is the $j$-th row of $\tilde\Sigma^{(i)}\scaledvs$.
Applying the Sherman-Morrison formula 
we obtain that the statistical leverage of the $j$-th row, $j\neq i$, of
$\tilde\Sigma^{(i)}\scaledvs$ can be written as follows
\begin{multline*}
\ell_j'  = 
 \hat x_j^T(\scaledvs^T\scaledvs)^{-1}x_j \\- \frac{(\alpha^2-1)\hat x_j^T(\scaledvs^T\scaledvs)^{-1}x_i
  x_i^T(\scaledvs^T\scaledvs)^{-1} \hat x_j}{1+ (\alpha^2-1)x_i^T(\scaledvs^T\scaledvs)^{-1}x_i}.
\end{multline*}
Since $\hat x_j=x_j$, we have
\[
\ell_j' = \ell_j -
\frac{(\alpha^2-1)(x_j^T(\scaledvs^T\scaledvs)^{-1}x_i)^2}{1+ (\alpha^2-1)\ell_i}.
\]
Finally, by definition, $x_i^T(\scaledvs^T\scaledvs)^{-1}x_i=\ell_i$,
so if $j=i$ we easily reach the expression for $\ell_i'$ stated in the lemma.
\end{proof}

%% file: caseur.tex
\section{Generalized Leverage Scores and Arbitrary Singular Subspaces}
\label{sec:caseur}
We will now generalize the result presented in
Section~\ref{sec:caseuk} to consider subspaces spanned by an
arbitrary subset of singular vectors.
We first define the \textit{generalized leverage~scores}.

\begin{definition}
Given a matrix $A$ of rank $\rho$ with singular value decomposition $A=U\Sigma V^T$, 
the generalized leverage score of the $i$-th column of $A$ 
with respect to the set $R\subseteq [\rho]$ is defined as $\|(V_R)_i\|_2^2$. 
\end{definition}
Instead of the connection between leverage scores and principal angles 
between $C$ and $U_k$, as in Section~\ref{sec:caseuk}, we are now interested
in the relationship between generalized leverage scores and the angles
between $C$ and $U_R$, for an arbitrary index set $R$.
In other words, we seek to bound $\|CC^+U_R\|_F^2$ in terms of
$\|V_R^TS\|_F^2$, using the matrices and notation introduced in Section~\ref{sec:caseuk}. This result will
lead to our approximation results, presented in
Section~\ref{sec:applications}.

Our analysis will be based in the following equalities, analogous to
Equations~(\ref{eq:sigma_vts}) and~(\ref{eq:vts_1}),
\begin{align}
\label{equr:sigma_vts}
\|CC^+U_R\|_F^2 
  & = \left \|\Sigma V^TS(\Sigma V^TS)^+\left(\begin{array}{c}I_R\\ 0 \end{array}\right) \right\|_F^2, \\
%
\label{equr:vts_2}
\|V_R^TS\|_F^2 
  & = \left \|V^TS(V^TS)^+\left(\begin{array}{c}I_R\\ 0 \end{array}\right) \right\|_F^2,
\end{align}
where $I_R$ is the matrix that ``picks'' the columns indexed by the set~$R$.

Again, we need to analyze the changes in the diagonal elements of
$V^TS(V^TS)^+$ upon left-multiplication of $V^TS$ by $\Sigma$.
The challenge now is that the entries of interest are no
longer the top ones. If we apply the previous reasoning, whereby we
analyzed a sequential application of the left-multiplication by $\Sigma$, 
it could be the case that some of the diagonal elements do indeed decrease. 
As a consequence, $\|CC^+U_R\|_F^2$ may become smaller than  $\|V_R^TS\|_F^2$, 
so an inequality like the one in Theorem~\ref{theorem_angles_ls} no longer holds.

Nevertheless, we can bound the extent of this decrease.
The next lemma 
provides such a bound,
and is essential for the final~result.
\begin{lemma}
  \label{lemma:urangle_bound}
  Consider a matrix and its singular value decomposition $A=U\Sigma
  V^T \in \mathbb R^{m\times n}$. Consider a column sampling matrix $S
  \in \mathbb R^{n\times r}$, and write $C=AS$. Consider an arbitrary
  index set $R$. We have 
  \[
  \|CC^+U_R\|_F^2 \geq \|V_R^TS\|_F^2 - \frac{\topsigma^2}{\botsigma^2}\left(|R|-\|V_R^TS\|_F^2\right),
  \]
  where $\topsigma=\max_{i\notin R}\sigma_i(A)$ and $\botsigma=\min_{i\in R}\sigma_i(A)$.
\end{lemma}

Before proving Lemma \ref{lemma:urangle_bound}, we will state the
Woodbury matrix identity, which we will employ as a technical crutch.
\begin{lemma}
\label{lemma:woodbury}
Woodbury matrix identity \cite{hager1989updating}.
  
Given an invertible matrix $\wmiA$, let $\wmiB=\wmiA+\wmiU\wmiV$. Then
\[
  \wmiB^{-1}=\wmiA^{-1}-\wmiA^{-1}\wmiU\left(I_{k}+\wmiV\wmiA^{-1}\wmiU\right)^{-1}\wmiV\wmiA^{-1}.
\]
\end{lemma}

We will also introduce some helpful notation.
We write $X=V^TS$, $H=XX^+$ and 
$\tilde H=\tilde \Sigma X(\tilde \Sigma X)^+$.
We define the set $M=\{i: i < \max R, i \notin R\}$.
We define $\botsigma=\min_{i\in R}\{\sigma_i(A)\}$,
$\topsigma=\max_{i\in M}\{\sigma_i(A)\}$ (or $\topsigma=0$ if $M=\emptyset$) and write
$\tilde\Sigma = \botsigma^{-1}\Sigma$.  

We use $x_i$ to denote the $i$-th row of $X$ (as a column vector).

Finally, for  $S\subset \mathbb N$ we
define
$\tilde \Sigma^S$ so that $\tilde \Sigma^S_{ii}=\tilde
\Sigma_{ii}$ if $i\in S$ and $\tilde \Sigma^S_{ii}=1$ otherwise;
and
$E^S=\sum_{i\in S}(\tilde\sigma_i^2-1)x_ix_i^T$.

Note that
$(\tilde\Sigma^S X)^T\tilde\Sigma^S X=X^TX+E^S$;
see proof of Lemma~\ref{lemma_rankone}.

%
\begin{proof}[Proof of Lemma~\ref{lemma:urangle_bound}]
Throughout this proof we assume that none of the singular vectors picked
(i.e., those indexed by the set $R$) belong to the nullspace of $A$.

Note that because of Equations~(\ref{equr:sigma_vts}) and~(\ref{equr:vts_2}), 
it will suffice to bound $\sum_{i\in R}\tilde H_{ii}$ 
in terms of  $\sum_{i\in R}H_{ii}$.

We first identify the scaling operations (i.e., the rows of $\tilde
\Sigma$) that cause the value of $\sum_{i\in R}H_{ii}$ to
decrease.

First, note that only rows with indices in $M$ or $R$ can have any
such effect (as argued in the proof of
Theorem~\ref{theorem_angles_ls}, multiplying a row of $V^TS$ by a value smaller
than $1$ will cause the rest of the diagonals of $H$ to increase). In the
case of a row $j\in R$, we have both a positive and a negative effect,
as $H_{jj}$ will increase and every $H_{ii}, i\in R, i\neq j$ will
decrease. 
By Lemma~\ref{lemma_rankone}, after scaling row $j$, we can write
\[
\ell_j(\tilde \Sigma^{(j)} X) = \ell_j(X) + \frac{(\alpha^2-1)\ell_j(X)-(\alpha^2-1)(\ell_j(X))^2}{1+(\alpha^2-1)\ell_j(X)}.
\]
On the other hand, the elements $H_{ii}$, for $i\in R$, $i\neq j$ will
experience the decrease indicated by Lemma~\ref{lemma_rankone}. 
Now, by Property~\ref{proph_norms} for projection matrices:
\[
\ell_j(X) \geq \sum_{i\in R}\ell_i^2(X).
\]
Therefore, after scaling any row $j\in R$, the net effect on
$\sum_{i\in R}H_{ii}$ will be positive.
It will thus be enough to bound the effect of scaling the rows indexed by $M$.
%
%
%
%
To accomplish this, we will use
Lemma~\ref{lemma:woodbury}. 
We write $\wmiA=X^TX$, $\wmiU=E^S$, and $\wmiV=I$. 
Lemma~\ref{lemma:woodbury} gives
\begin{align}
& \left((\tilde\Sigma^S\scaledvs)^T\tilde\Sigma^S\scaledvs\right)^{-1} \nonumber \\
& = (X^TX)^{-1} \nonumber \\&\quad - (X^TX)^{-1}E^S(I + (X^TX)^{-1}E^S)^{-1}(X^TX)^{-1} \nonumber \\ 
& =  I - E^S(I+E^S)^{-1}  , \label{eq:wmi_application}
\end{align}
because $X^TX=I$.
We are interested in analyzing how much $\tilde H_{jj}$ may decrease
with respect to $H_{jj}$, for any $j\in R$. Note that the diagonal
elements $H_{jj}$ may decrease each time we left-multiply $H$ by
$\tilde\Sigma^{(i)}$, for some $i \in M$. Thus, we can bound the total
decrease  for any $j\in R$ as follows:
\begin{align*}
\tilde H_{jj} &\geq
 x_j^T\left((\tilde\Sigma^{M}\scaledvs)^T\tilde\Sigma^{M}\scaledvs\right)^{-1} x_j
\\ &= x_j^Tx_j - x_j^TE^{M}(I+E^{M})^{-1}x_j 
\\ &= H_{jj} - x_j^T(\sum_{i\in M}(\tilde\sigma_i^2-1)x_ix_i^T)(I+E^{M})^{-1}x_j
\\ &= H_{jj} - \sum_{i\in M}(\tilde\sigma_i^2-1)H_{ji}x_i^T(I+E^{M})^{-1}x_j.
\end{align*}
Since $E^{M}$ is the sum of positive semidefinite matrices and thus positive semidefinite itself, the spectral norm of $(I+E^{M})^{-1}$ is at most
1. This means that $x_i^T(I+E^{M})^{-1}x_j \leq |x_i^Tx_j|$ and thus we can conclude that
\begin{align*}
  \tilde H_{jj} &\geq H_{jj} - \sum_{i\in M}(\tilde\sigma_i^2-1)H_{ji}|x_i^Tx_j|
  \\&= H_{jj} - \sum_{i\in M}(\tilde\sigma_i^2-1)H_{ji}^2.
\end{align*}
We can thus write
\begin{align*}
  \tilde H_{jj}&\geq H_{jj} - \sum_{i\in M}(\tilde\sigma_i^2-1)H_{ji}^2
  \\& \geq  H_{jj} - (\tilde\topsigma^2-1)\sum_{i \in M}H_{ji}^2
\\ & \geq  H_{jj} - (\tilde\topsigma^2-1)(\|H_{j,:}\|_2^2-H_{jj}^2)
\\ & =  H_{jj} - (\tilde\topsigma^2-1)(H_{jj}-H_{jj}^2)  & {\triangleright\text{ Property~\ref{proph_norms}}}
\\ & =  H_{jj} - (\tilde\topsigma^2-1)H_{jj}(1-H_{jj})
\\ & \geq  H_{jj} - \tilde\topsigma^2(1-H_{jj}).  & \triangleright\text{ $0\leq H_{jj}\leq 1$} 
\end{align*}
By summing over all rows of interest (that is, those in the set $R$), we can analyze the total change
attributable to the scaling of all rows:
\begin{align}
  \sum_{j\in R}\tilde H_{jj} &\geq  \sum_{j\in R}H_{jj} -
  \tilde\topsigma^2\sum_{j\in R}(1-H_{jj}) \nonumber
  \\ &= \sum_{j\in R}H_{jj} -
  \frac{\topsigma^2}{\botsigma^2}\sum_{j\in R}(1-H_{jj}) \label{eq:final_eq_lemmaur}.
\end{align}
The result now follows from the fact that $\sum_{j\in
  R}H_{jj}=\|V_R^TS\|_F^2$, established by Equation~(\ref{equr:vts_2}).
\end{proof}

Note that if $\sum_{i\in R} H_{jj} =\|V_R^TS\|_F^2 \geq |R|-\epsilon$,
from Equation~(\ref{eq:final_eq_lemmaur}) we easily get
\[
  \|CC^+U_R\|_F^2 \geq \|V_R^TS\|_F^2 - \frac{2\epsilon\topsigma^2}{\botsigma^2}.
  \]
We immediately obtain the following result.
\begin{theorem}
  \label{theorem_epsilon_ur}
  Consider a matrix $A$ and its singular value decomposition $A=U\Sigma
  V^T \in \mathbb R^{m\times n}$. Consider an arbitrary
  index set $R$ and a column
  sampling matrix $S \in \mathbb R^{n\times r}$ satisfying
  $\|V_R^TS\|_F^2\geq |R|-\frac{\epsilon\botsigma^2}{2\topsigma^2}$, and write $C=AS$.  
  Then 
  \[
  \|CC^+U_R\|_F^2 \geq \|V_R^TS\|_F^2 - \epsilon.
  \]  
\end{theorem}
Note that as an immediate corollary, we can replace
$\topsigma/\botsigma$ with the condition number of $A$. 

We remark briefly upon the insight revealed by this result. As
we descend into the depths of the singular spectrum, the leverage scores
provide an ever-weakening link between
columns and singular-vector subspaces. The extent of this decline is
quantified by the ratio $\topsigma/\botsigma$.

%% file: cssp.tex
\subsection{Column Subset Selection}
\label{section_cssp}

We apply our results to \gcssp (Problem~\ref{problem:gcssp}).
The only provable method we know for this formulation
was given by~\citet{bhaskara2016greedy}.

We propose Algorithm~\ref{cssp_algorithm} for Problem~\ref{problem:gcssp}, 
and show it enjoys approximation guarantees. 
The proof is in the Appendix.

\begin{theorem}
  Let $C=AS$, where $S$ is the matrix output by Algorithm~\ref{cssp_algorithm}. Then
  \[
  \|CC^+B\|_F^2 \geq (1-\epsilon)(1-\delta)\|B\|_F^2.
  \]
  \label{the:cssp}
\end{theorem}
\begin{algorithm}
    {\bf Input}: Target matrix $B \in \mathbb R^{m\times t}$, basis matrix
    $A=U\Sigma V^T$, $\epsilon,\delta \in \mathbb R$.
    
    \label{cssp:vectors}~1.~ Choose index set $R$ so that $\|U_R^TB\|_F^2 \geq (1-\delta)\|B\|_F^2$.
    
    \label{cssp:columns}~2.~ Output column-selection matrix $S$ based on generalized-leverage-scores ordering so that $\|V_R^TS\|_F^2 \geq |R|-\frac{\epsilon^2\botsigma^2}{8\topsigma^2}.$

    \caption{Generalized deterministic leverage score sampling for GCSS (Problem~\ref{problem:gcssp}).}
    \label{cssp_algorithm}
\end{algorithm}
This result is akin to that of~\citet{papailiopoulos2014provable}. 
In particular, it expands the guarantees to an arbitrary target matrix
$B$, as the result of~\citet{papailiopoulos2014provable} is limited to
the case $A=B$. Note that while their result is for the minimization
objective, one can obtain a bound for maximization by simple manipulations.

This result is also an alternative to that
of~\citet{bhaskara2016greedy}, which gives a
relative-error approximation to the optimum for the greedy algorithm. The number of
columns required is inversely proportional to the smallest singular value of the
optimal subset. Our result, on the contrary, does not rely on said
singular value, which is unknown and could be arbitrarily small, but on the easily computed decay of
the generalized leverage scores and the ratio of the chosen singular values. We will make this precise in Section~\ref{sec:number_of_columns}.


%% file: sparsecca.tex
\subsection{Sparse Canonical Correlation Analysis}
\label{sec:sparsecca}

Consider two matrices $A$ and $B$, whose rows correspond to mean-centered observations of
a collection of random variables. The problem of Canonical Correlation
Analysis~\cite{10.2307/2333955} is to find pairs of vectors, in the spaces spanned by the
columns of $A$ and $B$ respectively, that have maximal
correlation. Formally, the $i$-th canonical correlation between $A$
and $B$ can be defined as follows:
\begin{align*}
  \max_{u_i,v_i} & \quad u_i^TA^TBv_i
  \\ \text{such that} & \quad  \|Au_i\| = \|Bv_i\| = 1,
  \\ & \quad  u_i^Tu_j=v_i^Tv_j=0, \quad j=1,\dots, i-1.
\end{align*}

By decomposing $A=Q_A\Sigma_A V_A^T$ and $B=Q_B\Sigma_B V_B^T$, 
and optimizing over the vectors $\hat u=\Sigma_A V_A^Tu$ and $\hat v=\Sigma_B V_B^Tv$,
it is easy to see that computing
the singular value decomposition of $Q_A^TQ_B$ is equivalent to
finding the canonical correlations. In particular, these are given by
the resulting singular values, and are of course equal to the cosines of
the principal angles between $A$ and $B$.

\para{Sparsity.}
In high-dimensional settings, that is, when the matrices $A$ and $B$
have a large number of columns, one may be interested in knowing
whether a small number of variables account for a significant amount
of the canonical correlations. This can be accomplished by enforcing
sparsity into the vectors $u_i$ and $v_i$ when finding the $i$-th canonical
correlation. This is usually referred to as Sparse CCA.

We can show that we can expand our results from Section~\ref{section_cssp}
to solve Sparse CCA with approximation guarantees. 
The algorithm is similar to
Algorithm~\ref{cssp_algorithm}, as it can be interpreted as two-sided
column subset selection. We analyze the algorithm next.

\input{sparsecca_pseudocode}

\para{Analysis.}
We first compute $\|Q_A^TQ_B\|_F^2$, which is the total sum of canonical correlations
between $A$ and $B$. Our goal is to choose column subsets from both
matrices so as to preserve this quantity as much as possible. We first
proceed as in Algorithm~\ref{cssp_algorithm}, choosing columns of $A$
to approximate~$B$. 
We find a subset $C=AS$ such that
\[
\|CC^+Q_B\|_F^2 = \|Q_{AS}^TQ_B\|_F^2 \geq (1-\epsilon)(1-\delta)\|Q_A^TQ_B\|_F^2.
\]
Next, we do the same, but picking columns from $B$ to ap\-prox\-i\-mate
$AS$. We obtain a subset $C'=BS'$ such that
\begin{align*}
  \|C'C'^+Q_{AS}\|_F^2 &\geq (1-\epsilon)(1-\delta) \|Q_{AS}^TQ_B\|_F^2
  \\ &\geq (1-\epsilon)^2(1-\delta)^2\|Q_A^TQ_B\|_F^2.
\end{align*}
We easily derive the following result.

\begin{theorem}
\label{theorem_sparsecca}
  Given two matrices $A$ and $B$ with respective orthonormal bases $Q_A$ and $Q_B$,
  Algorithm~\ref{cca_algorithm} outputs two column-sampling matrices
  $S$ and $S'$ such that if $W$ and $W'$ are orthonormal bases of $AS$ and $BS'$
  respectively, then
  \[
  \|W^TW'\|_F^2 \geq (1-\epsilon)^2(1-\delta)^2\|Q_A^TQ_B\|_F^2.
  \]
\end{theorem}

%% file: sparsecca_pseudocode.tex
\begin{algorithm}[t]
    {\bf Input}: 
    Matrices  $A \in \mathbb R^{m\times n}$, $B \in \mathbb R^{m\times t}$, $\epsilon,\delta \in \mathbb R$.
    \begin{enumerate}
    \item Orthonormalize $A=Q_AR_A$ and $B=Q_BR_B$, and\\ compute $q=\|Q_A^TQ_B\|_F^2$.
    \item Compute the SVD of $A=U\Sigma V^T$.
    \item Choose index set $R$ so that $\|U_R^TQ_B\|_F^2 \geq (1-\delta)q$.
    \item Find a column-selection matrix $S$ so that\\ $\|V_R^TS\|_F^2 \geq |R|-\frac{\epsilon^2\botsigma^2}{8\topsigma^2}$.
    \item Compute the SVD of $B=U'\Sigma'V'^T$ and orthonormalize $AS=Q_{AS}R_{AS}$.
    \item Compute $q'=\|Q_{AS}^TQ_B\|_F^2$.
    \item Choose index set $T$ so that $\|U_T'^TQ_{AS}\|_F^2 \geq (1-\delta)q'$.
    \item Find a column-selection matrix $S'$ so that\\ $\|V_T'^TS'\|_F^2 \geq |T|-\frac{\epsilon^2\botsigma'^2}{8\topsigma'^2}$.
    \item Output matrices $S$ and $S'$.
    \end{enumerate}    
    \caption{Generalized deterministic leverage score sampling for
      Sparse CCA.}
        \label{cca_algorithm}
\end{algorithm}

%% file: practical.tex
\section{Complexity and Practical Aspects}
\label{sec:practical}

\para{Computational complexity.} As we have mentioned before, the running time of algorithms derived
from our approach is dominated by the cost of computing an SVD. Thus,
one can benefit from the extensive literature and off-the-shelf
software for computing SVD~\cite{golub1996matrix}. 
If one is willing
to trade accuracy for speed, it is possible to employ randomized
methods. In particular, given an input matrix of size $m\times n$, it
is possible to compute an approximate 
truncated SVD of rank-$k$ in time
$\mathcal{O}(mnk)$~\cite{martinsson2020randomized}. For the applications we
consider here, the precise accuracy of the approximation is not
important, as long as the order of the generalized
leverage scores is preserved.

\para{The effects of truncation.} The most straightforward way to
use the SVD efficiently is to truncate it, that is, to compute only
the leading $k$ singular vectors and values. Smaller values of $k$
will yield coarser approximations at improved speed and storage requisites. One
interesting aspect of our approach is that truncation does not
necessarily imply a loss of precision. In
Algorithm~\ref{cssp_algorithm}, for instance, note that only the
singular vectors up to $\botsigma$ are needed. The rest can be
discarded at no loss. In computationally constrained environments,
where computing a full SVD might be challenging, one can successively
compute more vectors, by increasing $k$, until $\|U_k^TB\|_F^2$
becomes large enough.

%% file: numberofcolumns.tex
\para{The number of selected columns.}
\label{sec:number_of_columns}
As shown in previous work~\cite{papailiopoulos2014provable},
when the rank-$k$ leverage scores follow a power-law decay,
$\mathrm{poly}(k,1/\epsilon)$ columns suffice to add up to $k-\epsilon$.

The result translates unchanged to
Theorem~\ref{theorem_angles_ls}, and applies to
Theorem~\ref{theorem_epsilon_ur} with small changes. To see this, it
is enough to observe that to obtain our
bound on the principal angles, we require
the generalized leverage scores to add up to
$|R|-{\epsilon\botsigma^2}/({2\topsigma^2})$. We apply the results of
Papailiopoulos directly.

In particular, if the generalized leverage scores  satisfy 
$\ell_i ={\ell_1}/{i^{(1+\eta)}}$, 
the number of required columns is
\[
c=\max\left \{ \left(\frac{4k\topsigma^2}{\epsilon\botsigma^2} \right
)^{\frac{1}{1+\eta}} \!\!-1,  ~ \left(\frac{4k\topsigma^2}{\epsilon\eta\botsigma^2} \right
)^{\frac{1}{\eta}} \!\! -1, \right \}.
\]

That is, the number of columns required to achieve
Theorem~\ref{the:cssp} for arbitrary constants $\epsilon, \delta$, in case of a power-law
leverage-score decay, is polynomial in ${k\topsigma^2}/{\botsigma^2}$.

Note that by replacing $\topsigma/\botsigma$ with the condition number of the input matrix
we obtain a result similar to one of~\citet{civril2012column}.

%% file: experiments.tex
\section{Experimental Results}

\begin{figure}[t]
\setlength{\tabcolsep}{-2pt}
  \begin{tabular}{cc}
    \includegraphics[width=.25\textwidth]{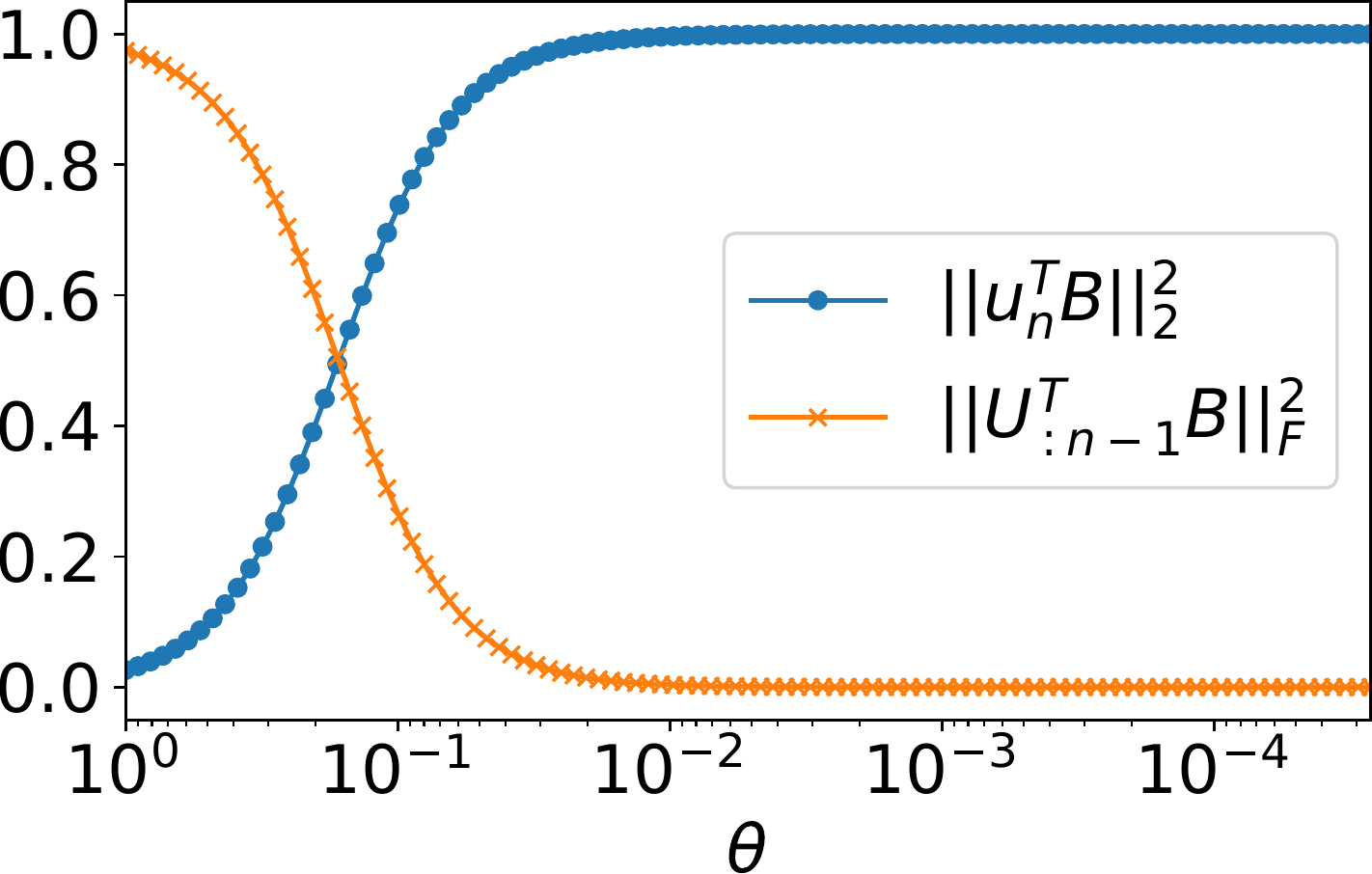} &        
    \includegraphics[width=.25\textwidth]{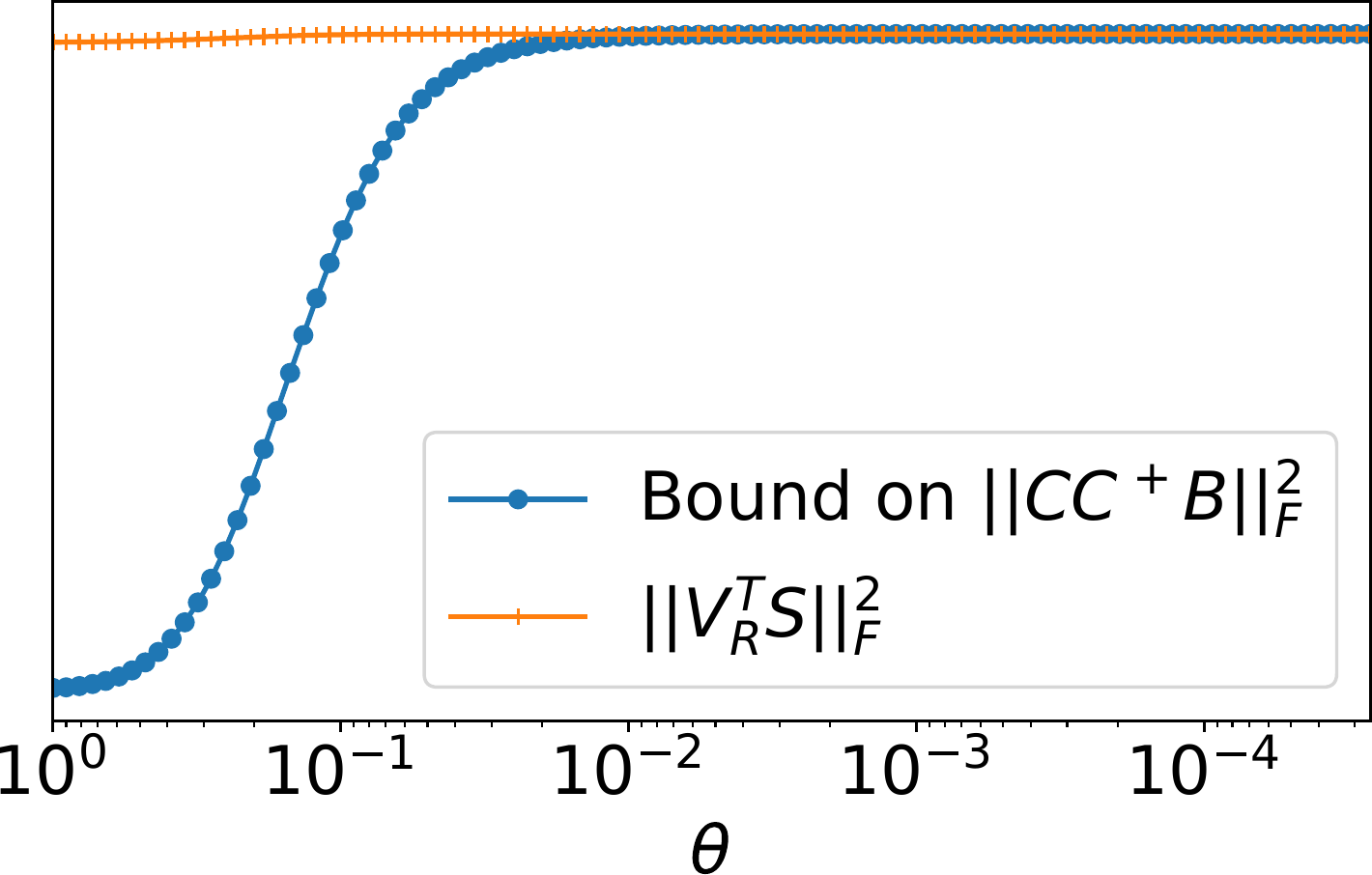}
  \end{tabular}
  
  \caption{Behavior of our method in an example of
    \citet{altschuler2016greedy}, as $\theta$ varies in [1,0). Left: the norm of $B$ when projected onto the space spanned by
      the last left singular vector and the rest. Right: Lower bound of Theorem~\ref{the:cssp}.}
    \label{fig:bhaskara}
\end{figure}

\begin{figure*}[t]
  \includegraphics[width=\textwidth]{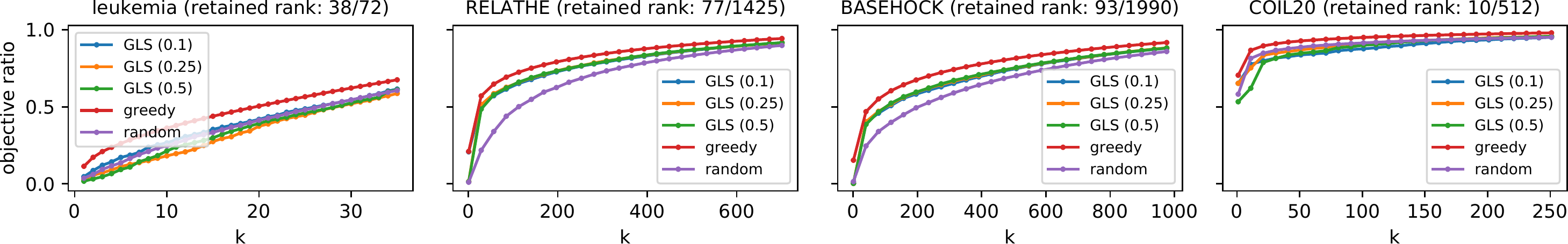} \\
  \includegraphics[width=\textwidth]{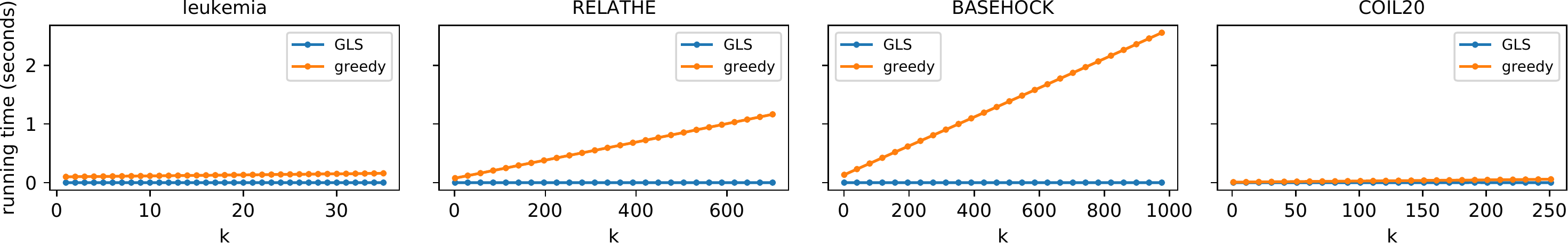}
  \caption{Performance analysis. \textit{Retained rank}: no. of singular values accounting for 75\% of $\|A\|_F^2$. \textit{Objective ratio}: $\frac{\|CC^+B\|_F^2}{\|AA^+B\|_F^2}$. 
  }
  \label{fig:performance}
\end{figure*}

We conduct experiments to gain further insight about our
results. Throughout this section, we focus on \gcssp and
Algorithm~\ref{cssp_algorithm}, which we will call \gls. 

A widely used algorithm in \gcssp literature is the \greedy algorithm, which iteratively selects the best column from $A$ to add to $C$, such that $\|CC^+B\|_F^2$ is maximized.
In practice, the \greedy algorithm can be implemented very efficiently
and often provides good results~\cite{farahat2011efficient}.
While we have observed superior
results from \greedy in terms of objective, our approach
does offer certain advantages. We present experimental results
to provide further insight on the behaviour of our algorithm, and to
help determine when it may be an alternative to \greedy.

\para{Can we outperform greedy?}
We consider an example by \citet{altschuler2016greedy}, where the
output of \greedy matches their quality guarantee.

Consider a set of orthogonal vectors $\{e_0, e_1, \dots, e_n\}$. We
build a matrix $A$ with columns $e_1, \theta e_0+ e_1$, and
$2\theta e_0 + e_j$ for $j\geq 2$. 
The target matrix $B$ is comprised only of  column~$e_0$. 
Even though $B$ can be expressed
using the first two columns of $A$, i.e., $e_1, \theta e_0+ e_1$, 
\greedy will pick $2\theta e_0 + e_{t+1}$ at iteration
$t$. To achieve a $(1-\epsilon)$-approximation of the optimum for
$k=2$, it will need more than $\frac{1}{2\theta^2\epsilon}$
columns. 

Does our approach fare any better? We analyze its behavior for varying
values of $\theta$, in an instantiation of the above example of size
$11\times 11$. In order to find the optimal
column subset, comprised of columns 1 and 2, our algorithm needs to
pick a singular vector subset $R$ such that the generalized leverage
scores are high for these two columns.

In Figure~\ref{fig:bhaskara} we plot key values for
$\theta\in(0,1]$. On the left, we show the norm of $B$ projected onto
  the last left singular vector $u_n$, as well as on the rest, $U_{:n-1}$. As $\theta$
  shrinks, the former rapidly approaches $1=\|B\|_F^2$, which
  indicates that $u_n$ is a good choice of $R$ (step~1 of Algorithm~\ref{cssp_algorithm}). 
  On the right
  we plot the sum of the generalized leverage scores for columns 1 and 2, $\|V_R^TS\|_F^2$, 
  and the lower bound on $\|CC^+B\|_F^2$, where $C=AS$, of
  Theorem~\ref{the:cssp}. The leverage scores associated to $u_n$
  always add up to almost 1. When $B$ projects well onto the space of
  $u_n$, the lower bound indicates that these two columns are a good
  choice. When $\theta$ is small, our algorithm will identify them for
  most values of $\epsilon,\delta$.

\para{Quality-efficiency trade-off.}
We evaluate the performance of Algorithm~\ref{cssp_algorithm} on a
collection of real datasets, obtained from a repository maintained by
Arizona State University for feature-selection
tasks.\footnote{\url{https://jundongl.github.io/scikit-feature/datasets.html}}
We split each dataset in two. The first half of the columns acts as
matrix $A$ (see def. of \gcssp), and the second as $B$.

We compare our method to
 \greedy
 and a uniformly-at-random baseline (averaged over 100 iterations).

We run a variation on Algorithm~\ref{cssp_algorithm}. To avoid large
values of the ratio $\frac{\topsigma}{\botsigma}$, we 
only retain singular vectors paired with the leading
 singular values
 accounting for at least 75\% of the total squared
norm.
As per Algorithm~\ref{cssp_algorithm}, one would then build
the set $R$ based on the choice of $\delta$. 
which however can make it difficult to control the size of $R$, and
thus the decay of leverage scores. We thus try setting the size
of $R$ to a number of values: $1/10, 1/4$ and $1/2$  of the
retained rank. For each resulting $R$ we compute generalized
leverage scores. We use the generalized-leverage-score ordering 
to build a column submatrix $C_k$ of size $k$, for a
range of $k$ (depending on the dataset size). We measure both
$\|C_kC_k^+V\|_F^2$ and the running time of both algorithms (\greedy and ours). 

The results are shown in
Figure~\ref{fig:performance}. \greedy is clearly superior in terms of
objective function. However, our approach is more efficient for large
values of $k$, as it requires essentially constant computation time
with respect to this parameter. On some datasets, the leverage scores
are not informative and a random choice may perform better in
expectation. This is consistent with our results, as in the absence of
sharp decay, there are no guarantees for small column subsets (see
 Section~\ref{sec:practical}). Our approach seems to perform
better on high-dimensional, high-rank data.

%% file: related.tex
\section{Related Work}

The concept of leverage scores can be traced back to \textit{statistical leverage}, 
long employed in statistics and the analysis of linear
regression~\cite{chatterjee1986influential}.  The statistical leverages
of a matrix, defined as the diagonal entries of the projector onto its
row space, determine the influence of each row on the solution of a
least-squares problem; hence ``leverage.'' It is not difficult to see
that the leverage scores of a matrix's columns correspond to the
statistical leverages of a related matrix; see Equation~(\ref{eq:vts_2}).

The idea of using leverage scores to find good column subsets has its
origins in randomized algorithms for fast matrix
multiplication and $\ell_2$
regression~\cite{drineas2006fast,drineas2006sampling}. These
methods sample rows or columns of the input matrices with
probability proportional to their norms, to then solve the given tasks on
the subsampled matrices with accuracy guarantees. It was later
observed that similar strategies could be employed to compute
 approximate matrix factorizations, such as
CUR~\cite{mahoney2009cur}. For this purpose, rows and
columns are sampled with probabilities proportional to their leverage
scores~\cite{drineas2008relative}. A similar approach yields
approximation algorithms for column subset
selection~\cite{boutsidis2009improved}. Leverage scores can also be used for the design of deterministic
algorithms~\cite{papailiopoulos2014provable}. The results
presented in our paper allows us to extend the approach of
\citet{papailiopoulos2014provable} to other problems, as we argue in Section~\ref{section_cssp}.
Superior algorithms were later proposed for column subset
selection~\cite{guruswami2012optimal,boutsidis2014near}. In addition, the greedy
algorithm has attracted attention of its
own~\cite{civril2012column,bhaskara2016greedy}, given its simplicity,
efficiency and practical effectiveness~\cite{farahat2011efficient}.

Recently, leverage scores have found applications in other areas, such as kernel
methods~\cite{scholkopf2002learning}. Despite their potential for
effective machine-learning algorithm design~\cite{cortes1995support},
kernels often require handling large matrices. For that reason,
sampling methods and approximate factorizations are often studied in
this context. Leverage-score sampling
has proven effective for this purpose~\cite{el2015fast,musco2016recursive,calandriello2017distributed,li2019towards,erdelyi2020fourier,liu2020randomb}. For
a comprehensive review of this and related methods, see the article of~\citet{liu2020random}.
For overviews on the use of leverage scores and on randomized
algorithms for matrix-related tasks, see the articles of~\citet{mahoney2011randomized}
and~\citet{martinsson2020randomized}.

%
%
%

%% file: appendix.tex
\section*{Appendix}
\subsection*{Proofs missing from the main text}

\input{proof_cssp}


  
\subsection*{Quality-efficiency tradeoff results for the rest of the datasets}

In Figures~\ref{fig:objective ratio} and~\ref{fig:running times} we
plot the objective ratio and the running time for a larger number of
datasets. \greedy remains the better option if measured by the
objective function. Our algorithm, \gls, performs better on high rank data sets,
in line with previous observations. In some cases uniform sampling
performs better. This reinforces our previous observation that for some datasets, the
leverage scores computed the chosen set $R$ are not informative. We note
here that more careful tuning may have served to achieve better
results, and thus leave a more thorough empirical evaluation for
future work.

With respect to running time (figure~\ref{fig:running times}), the
results are in line with what we would expect. With respect to $k$, \greedy scales
linearly with $k$ but \gls remains essentially constant.  

\begin{figure*}
    \includegraphics[width=\textwidth]{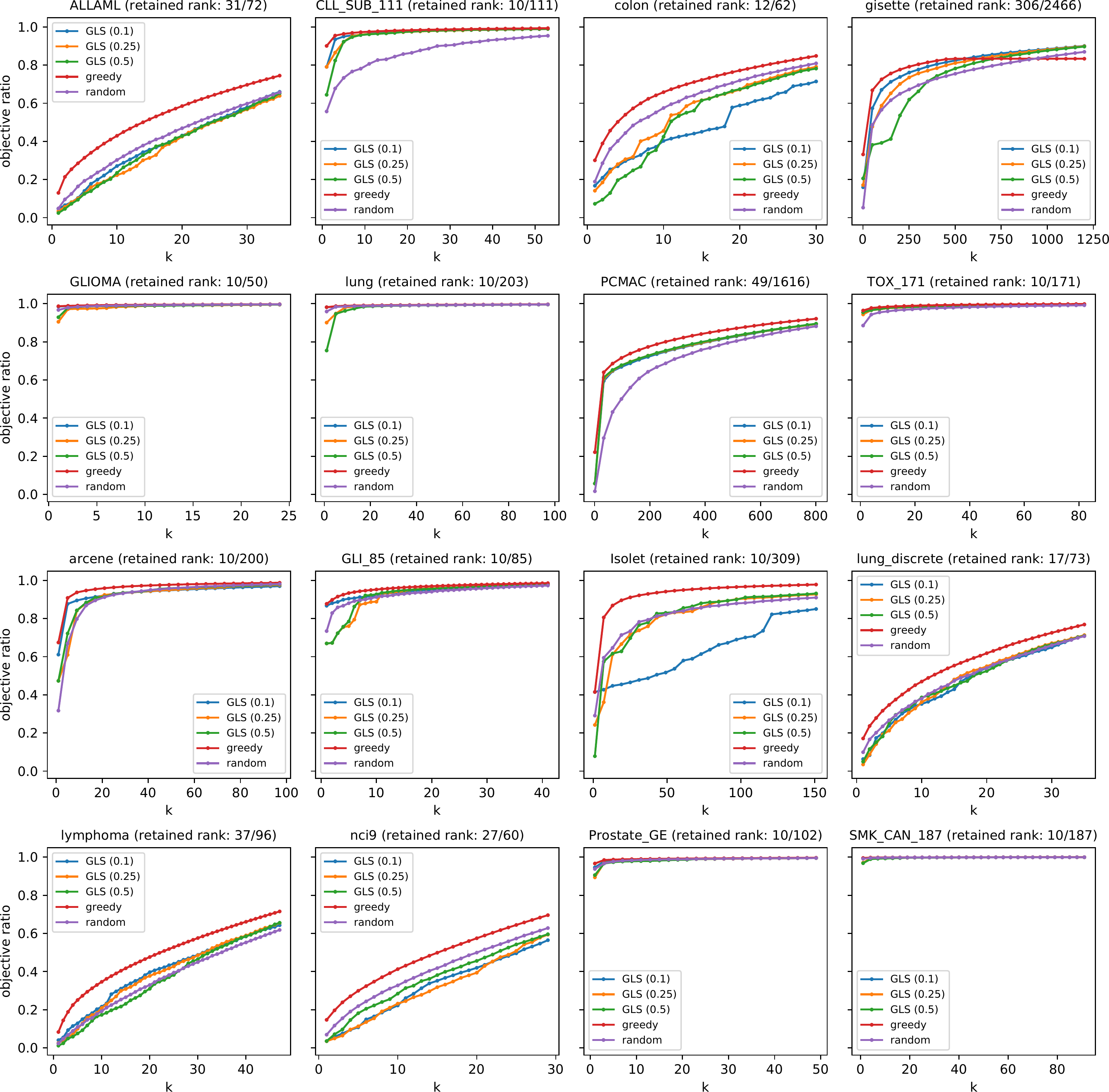}
	  \caption{Objective ratio for values of $k$ on the rest of the datasets}
  \label{fig:objective ratio}
\end{figure*}

\begin{figure*}
    \includegraphics[width=\textwidth]{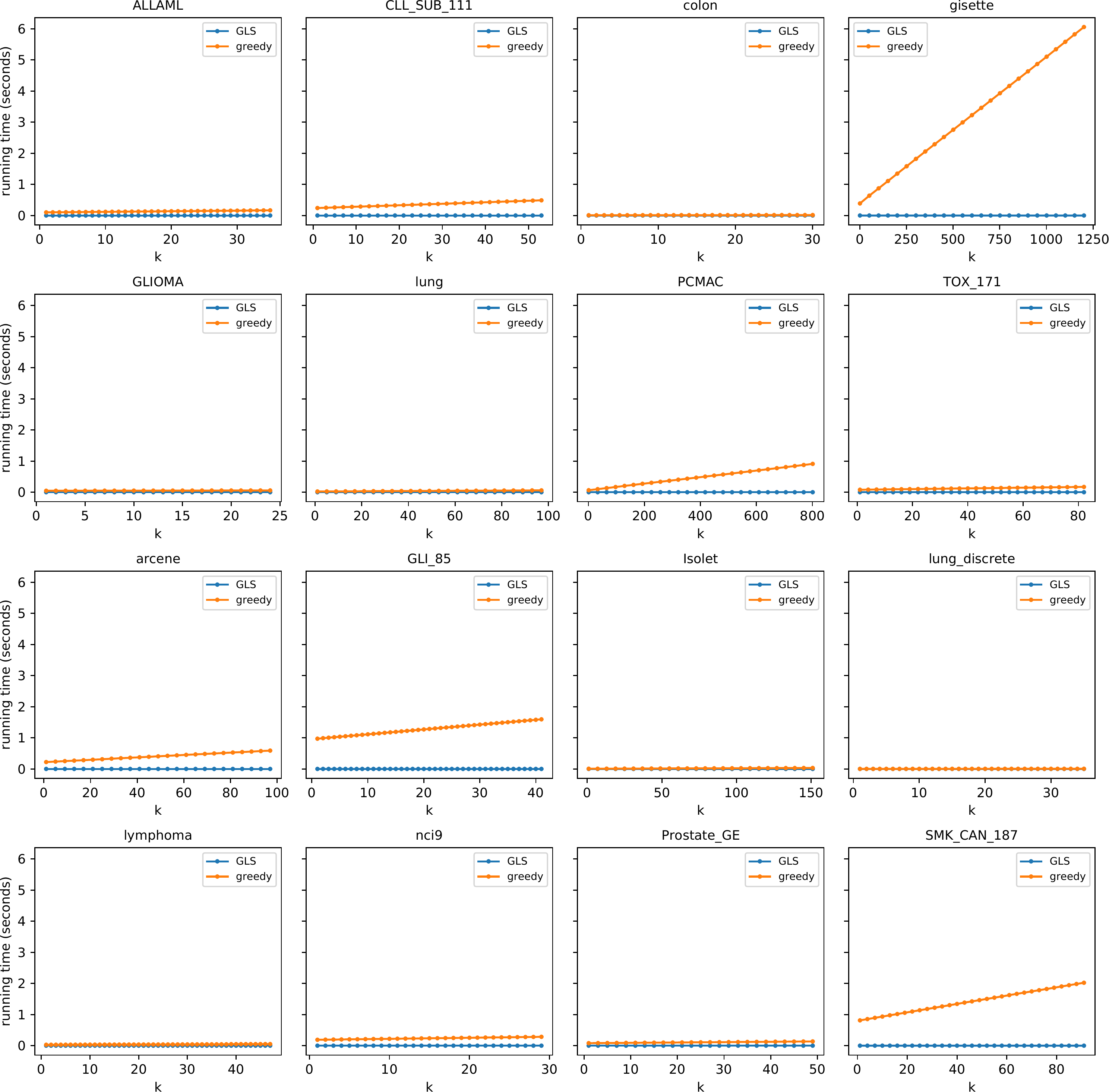}
	  \caption{Running times for values of $k$ on the rest of the datasets}
  \label{fig:running times}
\end{figure*}

%% file: proof_cssp.tex
\begin{proof}[Proof of Theorem~\ref{the:cssp}]
%
Throughout this proof, $\range(A)$ denotes the column space of a given matrix $A$.
  
We can write $\|U_R^TB\|_F^2 = \sum_i\|U_R^Tb_i\|_2^2 \geq
\sum_i(1-\delta_i)\|b_i\|_2^2$, for choices of $\delta_i$ such that
$\|U_R^Tb_i\|_2^2 \geq (1-\delta_i)\|b_i\|_2^2$, for $i=1,\dots,t$. This
implies that for every $i$ there is a unit vector $u \in \range(U_R)$
satisfying $(u^Tb_i)^2\geq (1 - \delta_i)\|b_i\|_2^2$.

On the other hand, Theorem~\ref{theorem_epsilon_ur} guarantees that
$\|CC^+U_R\|_F^2 \geq |R|-\epsilon^2/4$. This implies that for every
unit vector $u\in \range(U_R)$ there is a unit vector $w \in
\range(C)$ such that $(w^Tu)^2\geq 1-\epsilon^2/4$.

It can be shown~\cite{ordozgoiti2021insightful}
that for each $u \in \range(U_R)$, $w\in \range(C)$ and $b_i$, where
$u,w$ are unit~vectors:
\[
w^Tb_i \geq u^Tw u^Tb_i - \sqrt{\left(1-(u^Tw)^2\right)\left(\|b_i\|_2^2-(u^Tb_i)^2\right)}.
\]

We get $(u^Tw)^2 \geq 1-\epsilon^2/4$ and $(u^Tb_i)^2 \geq (1-\delta_i)\|b_i\|_2^2$.

Thus, for every $b_i$ we can choose $u,w$ so that
\begin{align*}
  (w^Tb_i)^2 &\geq \left(\sqrt{(1-\epsilon^2/4)(1-\delta_i)\|b_i\|_2^2} - \sqrt{\frac{\epsilon^2}{4}\delta_i\|b_i\|_2^2} \right)^2
  \\ &  = (1-\epsilon^2/4)(1-\delta_i)\|b_i\|_2^2 + \frac{\epsilon^2}{4}\delta_i\|b_i\|_2^2 \\&\quad- 2\sqrt{(1-\epsilon^2/4)(1-\delta_i)\frac{\epsilon^2}{4}\delta_i}\|b_i\|_2^2
  \\ &  = (1-\epsilon^2/4)(1-\delta_i)\|b_i\|_2^2 \\&\quad- \sqrt{\frac{\epsilon^2}{4}\delta_i}\|b_i\|_2^2\left(2\sqrt{(1-\epsilon^2/4)(1-\delta_i)} - 1\right)
  \\ &  \geq (1-\epsilon^2/4)(1-\delta_i)\|b_i\|_2^2 - \sqrt{\frac{\epsilon^2}{4}\delta_i}\|b_i\|_2^2
  \\ &  \geq (1-\epsilon^2/4)(1-\delta_i)\|b_i\|_2^2 - \sqrt{\frac{\epsilon^2}{4}}\|b_i\|_2^2
  \\ &  = (1-\epsilon^2/4)(1-\delta_i)\|b_i\|_2^2 - \frac{\epsilon}{2}\|b_i\|_2^2
\end{align*}

It follows from the previous exposition that for adequate choices of $w_i$ in the column space of $C$,
\begin{align*}
  \|CC^+B\|_F^2 &= \sum_{i=1}^t (w_i^Tb_i)^2
  \\&\geq \sum_{i=1}^t (1-\epsilon^2/4)(1-\delta_i)\|b_i\|_2^2 - \frac{\epsilon}{2}\|b_i\|_2^2
  \\&= (1-\epsilon^2/4)(1-\delta)\|B\|_F^2 - \frac{\epsilon}{2}\|B\|_F^2.
\end{align*}
Finally,
\begin{align*}
  (1- & \epsilon^2/4) (1-\delta)\|B\|_F^2 - \frac{\epsilon}{2}\|B\|_F^2
  \\ & = (1-\delta)\|B\|_F^2 - (1-\delta)\frac{\epsilon^2}{4}\|B\|_F^2 - \frac{\epsilon}{2}\|B\|_F^2
  \\ & \geq (1-\delta)\|B\|_F^2 - \frac{\epsilon}{2}\left(\frac{\epsilon}{2}\|B\|_F^2 +\|B\|_F^2\right)
  \\ & = (1-\delta)\|B\|_F^2 - \|B\|_F^2\frac{\epsilon}{2}\left(\frac{\epsilon}{2}+1\right)
  \\ & \geq (1-\delta)\|B\|_F^2\left(1- \frac{\epsilon}{2}\left(\frac{\epsilon}{2}+1\right)\right)
  \\ & \geq (1-\delta)\|B\|_F^2\left(1- 2\frac{\epsilon}{2}\right).
  \\ & = (1-\delta)\|B\|_F^2(1- \epsilon).
\end{align*}
\end{proof}

%% file: main.bbl
\begin{thebibliography}{34}
\providecommand{\natexlab}[1]{#1}
\providecommand{\url}[1]{\texttt{#1}}
\expandafter\ifx\csname urlstyle\endcsname\relax
  \providecommand{\doi}[1]{doi: #1}\else
  \providecommand{\doi}{doi: \begingroup \urlstyle{rm}\Url}\fi

\bibitem[Altschuler et~al.(2016)Altschuler, Bhaskara, Fu, Mirrokni,
  Rostamizadeh, and Zadimoghaddam]{altschuler2016greedy}
Altschuler, J., Bhaskara, A., Fu, G., Mirrokni, V., Rostamizadeh, A., and
  Zadimoghaddam, M.
\newblock Greedy column subset selection: New bounds and distributed
  algorithms.
\newblock In \emph{International conference on machine learning}, pp.\
  2539--2548. PMLR, 2016.

\bibitem[Bhaskara et~al.(2016)Bhaskara, Rostamizadeh, Altschuler,
  Zadimoghaddam, Fu, and Mirrokni]{bhaskara2016greedy}
Bhaskara, A., Rostamizadeh, A., Altschuler, J., Zadimoghaddam, M., Fu, T., and
  Mirrokni, V.
\newblock Greedy column subset selection: New bounds and distributed
  algorithms.
\newblock 2016.

\bibitem[Bj{\"o}rck \& Golub(1973)Bj{\"o}rck and Golub]{bjorck1973numerical}
Bj{\"o}rck, {\AA}. and Golub, G.~H.
\newblock Numerical methods for computing angles between linear subspaces.
\newblock \emph{Mathematics of computation}, 27\penalty0 (123):\penalty0
  579--594, 1973.

\bibitem[Boutsidis et~al.(2009)Boutsidis, Mahoney, and
  Drineas]{boutsidis2009improved}
Boutsidis, C., Mahoney, M.~W., and Drineas, P.
\newblock An improved approximation algorithm for the column subset selection
  problem.
\newblock In \emph{Proceedings of the Twentieth Annual ACM-SIAM Symposium on
  Discrete Algorithms}, SODA '09, pp.\  968–977, USA, 2009. Society for
  Industrial and Applied Mathematics.

\bibitem[Boutsidis et~al.(2014)Boutsidis, Drineas, and
  Magdon-Ismail]{boutsidis2014near}
Boutsidis, C., Drineas, P., and Magdon-Ismail, M.
\newblock Near-optimal column-based matrix reconstruction.
\newblock \emph{SIAM Journal on Computing}, 43\penalty0 (2):\penalty0 687--717,
  2014.

\bibitem[Calandriello et~al.(2017)Calandriello, Lazaric, and
  Valko]{calandriello2017distributed}
Calandriello, D., Lazaric, A., and Valko, M.
\newblock Distributed adaptive sampling for kernel matrix approximation.
\newblock In \emph{Artificial Intelligence and Statistics}, pp.\  1421--1429.
  PMLR, 2017.

\bibitem[Chatterjee \& Hadi(1986)Chatterjee and
  Hadi]{chatterjee1986influential}
Chatterjee, S. and Hadi, A.~S.
\newblock Influential observations, high leverage points, and outliers in
  linear regression.
\newblock \emph{Statistical science}, pp.\  379--393, 1986.

\bibitem[Civril(2014)]{civril2014column}
Civril, A.
\newblock Column subset selection problem is ug-hard.
\newblock \emph{Journal of Computer and System Sciences}, 80\penalty0
  (4):\penalty0 849--859, 2014.

\bibitem[Cortes \& Vapnik(1995)Cortes and Vapnik]{cortes1995support}
Cortes, C. and Vapnik, V.
\newblock Support-vector networks.
\newblock \emph{Machine learning}, 20\penalty0 (3):\penalty0 273--297, 1995.

\bibitem[Deshpande et~al.(2006)Deshpande, Rademacher, Vempala, and
  Wang]{deshpande2006matrix}
Deshpande, A., Rademacher, L., Vempala, S., and Wang, G.
\newblock Matrix approximation and projective clustering via volume sampling.
\newblock \emph{Theory of Computing}, 2\penalty0 (1):\penalty0 225--247, 2006.

\bibitem[Drineas et~al.(2006{\natexlab{a}})Drineas, Kannan, and
  Mahoney]{drineas2006fast}
Drineas, P., Kannan, R., and Mahoney, M.~W.
\newblock Fast monte carlo algorithms for matrices i: Approximating matrix
  multiplication.
\newblock \emph{SIAM Journal on Computing}, 36\penalty0 (1):\penalty0 132--157,
  2006{\natexlab{a}}.

\bibitem[Drineas et~al.(2006{\natexlab{b}})Drineas, Mahoney, and
  Muthukrishnan]{drineas2006sampling}
Drineas, P., Mahoney, M.~W., and Muthukrishnan, S.
\newblock Sampling algorithms for l 2 regression and applications.
\newblock In \emph{Proceedings of the seventeenth annual ACM-SIAM symposium on
  Discrete algorithm}, pp.\  1127--1136, 2006{\natexlab{b}}.

\bibitem[Drineas et~al.(2008)Drineas, Mahoney, and
  Muthukrishnan]{drineas2008relative}
Drineas, P., Mahoney, M.~W., and Muthukrishnan, S.
\newblock Relative-error cur matrix decompositions.
\newblock \emph{SIAM Journal on Matrix Analysis and Applications}, 30\penalty0
  (2):\penalty0 844--881, 2008.

\bibitem[El~Alaoui \& Mahoney(2015)El~Alaoui and Mahoney]{el2015fast}
El~Alaoui, A. and Mahoney, M.~W.
\newblock Fast randomized kernel ridge regression with statistical guarantees.
\newblock In \emph{NIPS}, 2015.

\bibitem[Erd{\'e}lyi et~al.(2020)Erd{\'e}lyi, Musco, and
  Musco]{erdelyi2020fourier}
Erd{\'e}lyi, T., Musco, C., and Musco, C.
\newblock Fourier sparse leverage scores and approximate kernel learning.
\newblock \emph{arXiv preprint arXiv:2006.07340}, 2020.

\bibitem[Farahat et~al.(2011)Farahat, Ghodsi, and Kamel]{farahat2011efficient}
Farahat, A.~K., Ghodsi, A., and Kamel, M.~S.
\newblock An efficient greedy method for unsupervised feature selection.
\newblock In \emph{2011 IEEE 11th International Conference on Data Mining},
  pp.\  161--170. IEEE, 2011.

\bibitem[Golub \& Van~Loan(1996)Golub and Van~Loan]{golub1996matrix}
Golub, G.~H. and Van~Loan, C.~F.
\newblock Matrix computations. edition, 1996.

\bibitem[Guruswami \& Sinop(2012)Guruswami and Sinop]{guruswami2012optimal}
Guruswami, V. and Sinop, A.~K.
\newblock Optimal column-based low-rank matrix reconstruction.
\newblock In \emph{Proceedings of the twenty-third annual ACM-SIAM symposium on
  Discrete Algorithms}, pp.\  1207--1214. SIAM, 2012.

\bibitem[Hager(1989)]{hager1989updating}
Hager, W.~W.
\newblock Updating the inverse of a matrix.
\newblock \emph{SIAM review}, 31\penalty0 (2):\penalty0 221--239, 1989.

\bibitem[Hardoon \& Shawe-Taylor(2011)Hardoon and
  Shawe-Taylor]{hardoon2011sparse}
Hardoon, D.~R. and Shawe-Taylor, J.
\newblock Sparse canonical correlation analysis.
\newblock \emph{Machine Learning}, 83\penalty0 (3):\penalty0 331--353, 2011.

\bibitem[Hotelling(1936)]{10.2307/2333955}
Hotelling, H.
\newblock Relations between two sets of variates.
\newblock \emph{Biometrika}, 28\penalty0 (3/4):\penalty0 321--377, 1936.
\newblock ISSN 00063444.
\newblock URL \url{http://www.jstor.org/stable/2333955}.

\bibitem[Li et~al.(2019)Li, Ton, Oglic, and Sejdinovic]{li2019towards}
Li, Z., Ton, J.-F., Oglic, D., and Sejdinovic, D.
\newblock Towards a unified analysis of random fourier features.
\newblock In \emph{International Conference on Machine Learning}, pp.\
  3905--3914. PMLR, 2019.

\bibitem[Liu et~al.(2020{\natexlab{a}})Liu, Huang, Chen, and
  Suykens]{liu2020random}
Liu, F., Huang, X., Chen, Y., and Suykens, J.~A.
\newblock Random features for kernel approximation: A survey on algorithms,
  theory, and beyond.
\newblock \emph{arXiv preprint arXiv:2004.11154}, 2020{\natexlab{a}}.

\bibitem[Liu et~al.(2020{\natexlab{b}})Liu, Huang, Chen, Yang, and
  Suykens]{liu2020randomb}
Liu, F., Huang, X., Chen, Y., Yang, J., and Suykens, J.
\newblock Random fourier features via fast surrogate leverage weighted
  sampling.
\newblock In \emph{Proceedings of the AAAI Conference on Artificial
  Intelligence}, volume~34, pp.\  4844--4851, 2020{\natexlab{b}}.

\bibitem[Mahoney(2011)]{mahoney2011randomized}
Mahoney, M.~W.
\newblock Randomized algorithms for matrices and data.
\newblock \emph{arXiv preprint arXiv:1104.5557}, 2011.

\bibitem[Mahoney \& Drineas(2009)Mahoney and Drineas]{mahoney2009cur}
Mahoney, M.~W. and Drineas, P.
\newblock Cur matrix decompositions for improved data analysis.
\newblock \emph{Proceedings of the National Academy of Sciences}, 106\penalty0
  (3):\penalty0 697--702, 2009.

\bibitem[Martinsson \& Tropp(2020)Martinsson and
  Tropp]{martinsson2020randomized}
Martinsson, P.-G. and Tropp, J.~A.
\newblock Randomized numerical linear algebra: Foundations and algorithms.
\newblock \emph{Acta Numerica}, 29:\penalty0 403--572, 2020.

\bibitem[Musco \& Musco(2016)Musco and Musco]{musco2016recursive}
Musco, C. and Musco, C.
\newblock Recursive sampling for the nystr$\backslash$" om method.
\newblock \emph{arXiv preprint arXiv:1605.07583}, 2016.

\bibitem[Ordozgoiti et~al.(2021)Ordozgoiti, Pai, and
  Ko{\l}czy{\'n}ska]{ordozgoiti2021insightful}
Ordozgoiti, B., Pai, S., and Ko{\l}czy{\'n}ska, M.
\newblock Insightful dimensionality reduction with very low rank variable
  subsets.
\newblock In \emph{Proceedings of the Web Conference 2021}, pp.\  3066--3075,
  2021.

\bibitem[Papailiopoulos et~al.(2014)Papailiopoulos, Kyrillidis, and
  Boutsidis]{papailiopoulos2014provable}
Papailiopoulos, D., Kyrillidis, A., and Boutsidis, C.
\newblock Provable deterministic leverage score sampling.
\newblock In \emph{Proceedings of the 20th ACM SIGKDD international conference
  on Knowledge discovery and data mining}, pp.\  997--1006, 2014.

\bibitem[Sch{\"o}lkopf et~al.(2002)Sch{\"o}lkopf, Smola, Bach,
  et~al.]{scholkopf2002learning}
Sch{\"o}lkopf, B., Smola, A.~J., Bach, F., et~al.
\newblock \emph{Learning with kernels: support vector machines, regularization,
  optimization, and beyond}.
\newblock MIT press, 2002.

\bibitem[Shitov(2017)]{shitov2017column}
Shitov, Y.
\newblock Column subset selection is np-complete.
\newblock \emph{arXiv preprint arXiv:1701.02764}, 2017.

\bibitem[Uurtio et~al.(2017)Uurtio, Monteiro, Kandola, Shawe-Taylor,
  Fernandez-Reyes, and Rousu]{uurtio2017tutorial}
Uurtio, V., Monteiro, J.~M., Kandola, J., Shawe-Taylor, J., Fernandez-Reyes,
  D., and Rousu, J.
\newblock A tutorial on canonical correlation methods.
\newblock \emph{ACM Computing Surveys (CSUR)}, 50\penalty0 (6):\penalty0 1--33,
  2017.

\bibitem[Çivril \& Magdon-Ismail(2012)Çivril and
  Magdon-Ismail]{civril2012column}
Çivril, A. and Magdon-Ismail, M.
\newblock Column subset selection via sparse approximation of svd.
\newblock \emph{Theoretical Computer Science}, 421:\penalty0 1 -- 14, 2012.
\newblock ISSN 0304-3975.
\newblock \doi{https://doi.org/10.1016/j.tcs.2011.11.019}.
\newblock URL
  \url{http://www.sciencedirect.com/science/article/pii/S0304397511009388}.

\end{thebibliography}
